\documentclass[a4paper,twoside]{article} 
\usepackage{times}
\usepackage{helvet}
\usepackage{courier}

\usepackage{graphicx}
\usepackage{amsmath,amsfonts}
\usepackage{amsthm}
\graphicspath{{figures/}}
\newtheorem{definition}{Definition}
\newtheorem{theorem}{Theorem}
\usepackage{algorithmic,algorithm}

\usepackage{SCITEPRESS}
\begin{document}

\title{Deceptive AI Explanations - Creation and Detection}

\author{\authorname{Johannes Schneider\sup{1}, Christian Meske\sup{2} and Michalis Vlachos\sup{3}}
\affiliation{\sup{1} University of Liechtenstein, Vaduz, Liechtenstein}
\affiliation{\sup{2} University of Bochum, Bochum, Germany}
\affiliation{\sup{3} University of Lausanne, Lausanne, Switzerland}
\email{johannes.schneider@uni.li, christian.meske@ruhr-uni-bochum.de, michalis.vlachos@unil.ch}
}


\abstract{ 
Artificial intelligence (AI) comes with great opportunities but can also pose significant risks. Automatically generated explanations for decisions can increase transparency and foster trust, especially for systems based on automated predictions by AI models. However, given, e.g., economic incentives to create dishonest AI, to what extent can we trust explanations? To address this issue, our work investigates how AI models (i.e., deep learning, and existing instruments to increase transparency regarding AI decisions) can be used to create and detect deceptive explanations. As an empirical evaluation, we focus on text classification and alter the explanations generated by GradCAM, a well-established explanation technique in neural networks. Then, we evaluate the effect of deceptive explanations on users in an experiment with 200 participants. Our findings confirm that deceptive explanations can indeed fool humans. However, one can deploy machine learning (ML) methods to detect seemingly minor deception attempts with accuracy exceeding 80\% given sufficient domain knowledge. Without domain knowledge, one can still infer inconsistencies in the explanations in an unsupervised manner, given basic knowledge of the predictive model under scrutiny.
}
\onecolumn \maketitle \normalsize \setcounter{footnote}{0} \vfill

\section{\uppercase{Introduction}}
AI can be used to increase wealth and well-being globally. However, the potential uses of AI cause concerns. For example, because of the limited moderation of online content, attempts at deception proliferate. Online media struggle against the plague of ``fake news", and e-commerce sites spend considerable effort in detecting deceptive product reviews (see \cite{wu2020} for a survey). Marketing strategies exist that consider the creation of fake reviews to make products appear better or to provide false claims about product quality \cite{adelani2019generating}.

There are multiple reasons why to provide ``altered" explanations of a predictive system. Truthful explanations might allow to re-engineer the logic of the AI system, i.e., leak intellectual property. Decision-makers might also deviate from suggested AI decisions at will. For example, a bank employee might deny a loan to a person she dislikes claiming an AI model's recommendation as to the reason, supported by a made-up explanation (irrespective of the actual recommendation of the system). AI systems may perform better when using information that should not be used but is available. For example, private information on a person's health condition might be used by insurances to admit or deny applicants. Even though this is forbidden in some countries, the information is still very valuable in estimating expected costs of the applicant if admitted. Product suggestions delivered through recommender systems are also commonly accompanied by explanations~\cite{fusco2019reconet} in the hope of increasing the likelihood of sales. Companies have an incentive to provide explanations that lure customers into sales irrespective of their truthfulness. As such, there are incentives to build systems that utilize such information but hide its use. That is, ``illegal" decision criteria are used, but they are omitted from explanations requested by authorities or even citizens. In Europe, the GDPR law grants rights to individuals to get explanations of decisions made in an automated manner. 

The paper contributes in the area of empirical and formal analyses of deceptive AI explanations. Our empirical analysis, including a user study, shows in alignment with prior work that deceptive AI explanations can mislead people. Our formal analysis sets forth some generic conditions under which detection of deceptive explanations is possible. We show that domain knowledge is required to detect certain forms of deception that might not be available to explainees (the recipients of explanations). Our supervised and unsupervised detection marks one of the first steps in the quest against deceptive explanations. They highlight that while detecting deception is often possible, success depends on multiple factors such as type of deception, availability of domain knowledge and basic knowledge of the deceptive system.

\section{\uppercase{Problem Definition}}
We consider classification systems that are trained using a labeled dataset $\mathcal{D}=\{(X,Y)\}$ with two sources of deception: model decisions and explanations. A model $M$ maps input $X \in S$ to an output $Y$, where $S$ is the set of all possible inputs. To measure the level of deception, we introduce a reference (machine learning (ML)) model $M^*$ and a reference explanation method $H^*$. In practice, $M^*$ might be a deep learning model and $H^*$ a commonly used explainability method such as LIME or SHAP. That is, $H^*$ might not be perfect. Still, we assume that the explainee trusts it, i.e. she understands its behavior and in what ways explanations differ from "human" reasoning. The model $M^*$ is optimized with a benign objective, i.e. maximizing accuracy. We assume that $M^*$ is not optimized to be deceptive. However, model $M^*$ might not be fair and behave unethically. A deceiver might pursue other objectives than those used for $M^*$ leading to the deceiver's model $M^D$. The model $M^D$ might simply alter a few decisions of $M^*$ using simple rules or it might be a completely different model. 
A (truthful) explainability method $H(X,Y,M)$ receives input $X$, class label $Y$ and model $M$ to output an explanation. For the reference explanation method $H^*$, this conforms to provide a best-effort, ideally a truthful, reasoning, why model $M$ would output class $Y$. The deceiver's method $H^D$ might deviate from $H^*$ using arbitrary information. It returns $H^D(X)$, where the exact deception procedure is defined in context.
An explainee (the recipient of an explanation) obtains for an input $X$, a decision $M^D(X)$ and an explanation $H^D(X)$. The decision is allegedly from $M^*$ and the explanation allegedly from $H^*$ and truthful to the model $M^D$ providing the decision. Thus, an explainee should be lured into believing that $M^*(X)=M^D(X)$ and $H^D(X)=H^*(X,M^D(X),M^D)$. However, the deceiver's model might not output $M^D(X)=M^*(X)$ and a deceiver might choose an explainability method $H^D$ that differs from $H^*$ or she might explain a different class $Y$. This leads to four scenarios (see Figure \ref{fig:scenarios}). We write $H^*(X):=H^*(X,M^D(X),M^D)$.

The goal of a deceiver is to construct an explanation so that the explainee is neither suspicious about the decision in case it is not truthful to the model $M^D$, ie. $M^D(X)\neq M^*(X)$, nor about the explanation $H^D(X)$ if it deviates from  $H^*(X,M^D(X),M^*)$. Thus, an explanation might be used to hide an unfaithful decision to the model or it might be used to convey a different decision-making process than occurs in $M^D$. 

\begin{figure*}
  \centering
  \includegraphics[width=0.75\linewidth]{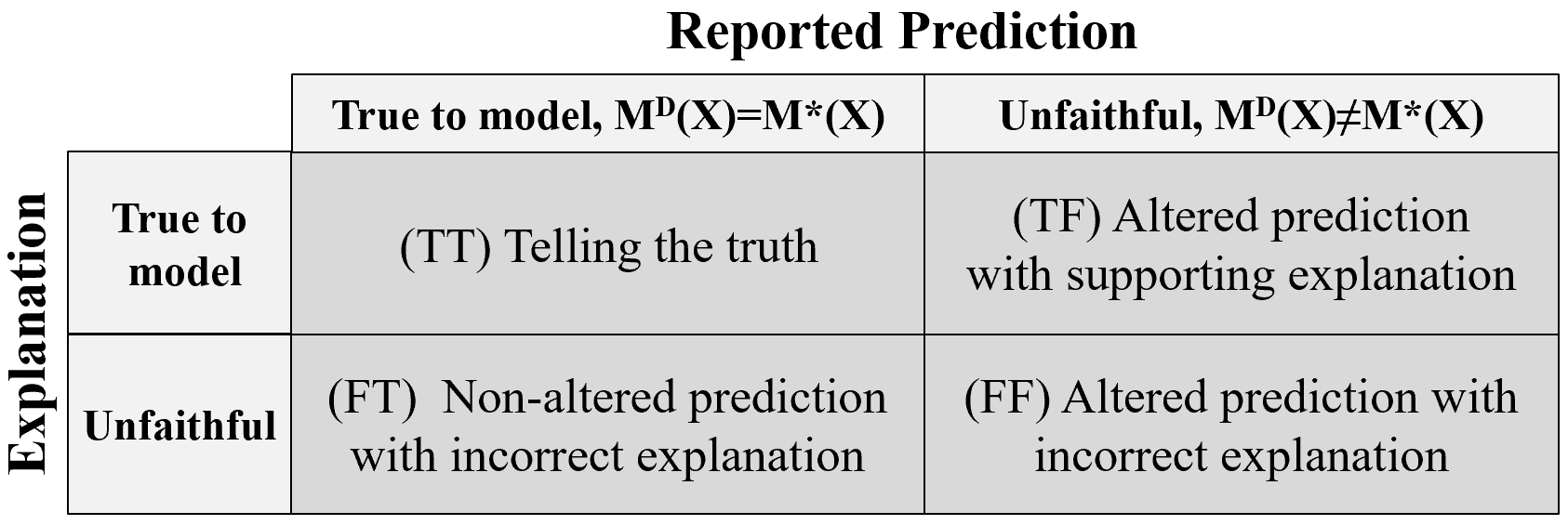}
  \caption{Scenarios for reported predictions and explanations}
  \label{fig:scenarios}
\end{figure*}

An input $X$ consists of values for $n$ features, $\mathcal{F}=\{i|i=1\ldots n\}$, where each feature $i$ has a single value $x_i \in V_i$ of a set of feasible values $V_i$. For example, an input $X$ can be a text document such as a job application, where each feature $i$ is a word specified by a word id $x_i$. Documents $X \in S$ are extended or cut to a fixed length $n$. ML models learn (a hierarchy of) features. Explaining in terms of learnt features is challenging since they are not easily mapped to unique concepts that are humanly understandable. Thus, we focus on explanations that assign \textit{relevance scores} to features $\mathcal{F}$ of an input $X$. Formally, we consider explanations $H$ that output a value $H_i(X,Y,M)$ for each feature $i \in F$. Where $H_i>0$ implies that feature $i$ with value $x_i$ is supportive of decision $Y$. A value of zero implies no dependence of $i$ on the decision $Y$. $H_i<0$ shows that feature $i$ is indicative of another decision.

\section{\uppercase{Measuring Explanation Faithfulness}} \label{sec:expfaith}
We measure faithfulness of an explanation using two metrics, namely \textit{decision fidelity} and \textit{explanation fidelity}.
\paragraph*{Decision fidelity.} It amounts to the standard notion of quantifying whether input $X$ and explanation $H^D(X)$ on their own allow deriving the correct decision $Y=M^*(X)$~\cite{schneider2019pers}. Therefore, if explanations indicate multiple outputs or an output different from $Y$, this is hardly possible. Decision fidelity $f_D$ can be defined as the loss when predicting the outcome using some classifier $g$ based on the explanation only, or formally:
\begin{small} 
\begin{equation} \label{def:decfid}
  f_D(X)=-L(g(X,H^D(X)),Y)
\end{equation}
\end{small}
The loss might be defined as 0 if $g(X,H^D(X))=Y$ and 1 otherwise. We assume that the reference explanations $H^*(X,M^*(X),M^*)$ results in minimum loss, i.e., maximum decision fidelity. (Large) decision fidelity does not require that an explanation contains all relevant features used to derive the decision $M^D(X)$. For example, in a hiring process, gender might influence the decision, but for a particular candidate other factors, such as qualification, social skills etc., are dominant and on their own unquestionably lead to a hiring decision.

\paragraph*{Explanation fidelity.} This refers to the overlap of the (potentially deceptive) explanation $H^D(X)$ and the reference explanation $H^*(X,M^D(X),M^D)$ for an input $X$ and reported decision $M^D(X)$. Any mismatch of a feature in the two explanations lowers explanation fidelity. It is defined as:
\begin{small}
\begin{equation}
  f_O(X)=1-\frac{\|H^*(X,M^D(X),M^D)-H^D(X)\|}{\|H^*(X,M^D(X),M^D)\|}
\end{equation}
\end{small}
Even if the decision $M^D(X)$ is non-truthful to the model, i.e., $M^D(X)\neq M^*(X)$, explanation fidelity might be large if the explanation correctly outputs the reasoning that would lead to the reported decision. If the reported decision is truthful, i.e., $M^D(X)= M^*(X)$, there seems to be an obvious correlation between decision- and explanation fidelity. But any arbitrarily small deviation of explanation fidelity from the maximum of 1 does not necessarily ensure large decision fidelity and vice versa. For example, assume that an explanation from $H^D$ systematically under- or overstates the relevance of features, i.e. $H^D(X)_i=H^*(X)_i\cdot c_i$ with arbitrary $c_i>0$ and $c_i\neq 1$. For $c_i$ differing significantly from 1, this leads to explanations that are far from the truth, which is captured by low explanation fidelity. However, decision fidelity might yield the opposite picture, i.e., maximum decision fidelity, since a classifier $g$ (Def. \ref{def:decfid}) trained on inputs $(X,H^D(X))$ with labels $M^D(X)$, might learn the coefficients $c_i$ and predict labels without errors.

\begin{figure}
  \centering
  \includegraphics[width=1.03\linewidth]{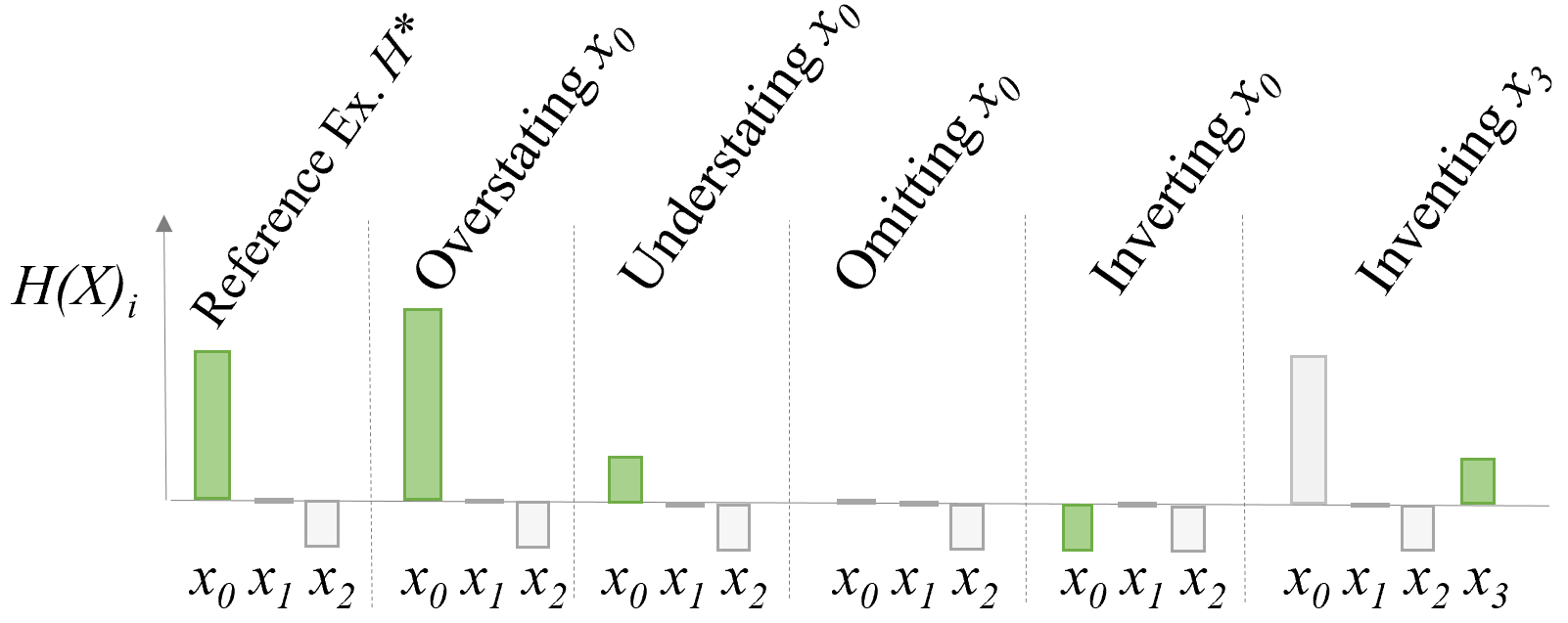}
  \caption{Deviations from (trusted) reference explanation}
  \label{fig:howtolie}
\end{figure}

Explanation fidelity captures the degree of deceptiveness of explanations from $H^D$ by aggregating the differences of its relevances of features and those of the reference explanations. When looking at individual features from a layperson's perspective, deception can arise due to over- and understating the feature's relevance or even fabricating features (see Figure~\ref{fig:howtolie}). Omission and inverting of features can be viewed as special cases of over- and understating. In this work, we do not consider feature fabrication.

\section{\uppercase{Creation of Deceptive Explanations}}
We first discuss goals a deceiver might pursue using deceptive explanations, followed by how deceptive explanations can be created using these goals in mind.
\paragraph*{Purposes of Deceptive Explanation} include:\\
i)  Convincing the explainee of an incorrect prediction, i.e. that a model decided $Y$ for input $X$ although the model's output is $M^D(X)$ with $Y\neq M^D(X)$. For example, a model $M^*$ in health-care might predict the best treatment for a patient trained on historical data $\mathcal{D}$. A doctor might change the prediction. She might provide the best treatment for well-paying (privately insured) patients and choose a treatment that minimizes her effort and costs for other patients. \\
ii) Providing an explanation that does not accurately capture model behavior without creating suspicion. An incorrect explanation will manifest in low decision fidelity and explanation fidelity. It involves hiding or overstating the importance of features in the decision process (Figure~\ref{fig:howtolie}) with more holistic goals such as:\\
  a) Omission: Hiding that decisions are made based on specific attributes such as gender or race to prevent legal consequences or a loss in reputation.\\
  b) Obfuscation: Hiding the decision mechanism of the algorithm to protect intellectual property.\\
The combination of (i) and (ii) leads to the four scenarios shown in Figure~\ref{fig:scenarios}.

\begin{figure*}
  \centering
  \includegraphics[width=0.65\linewidth]{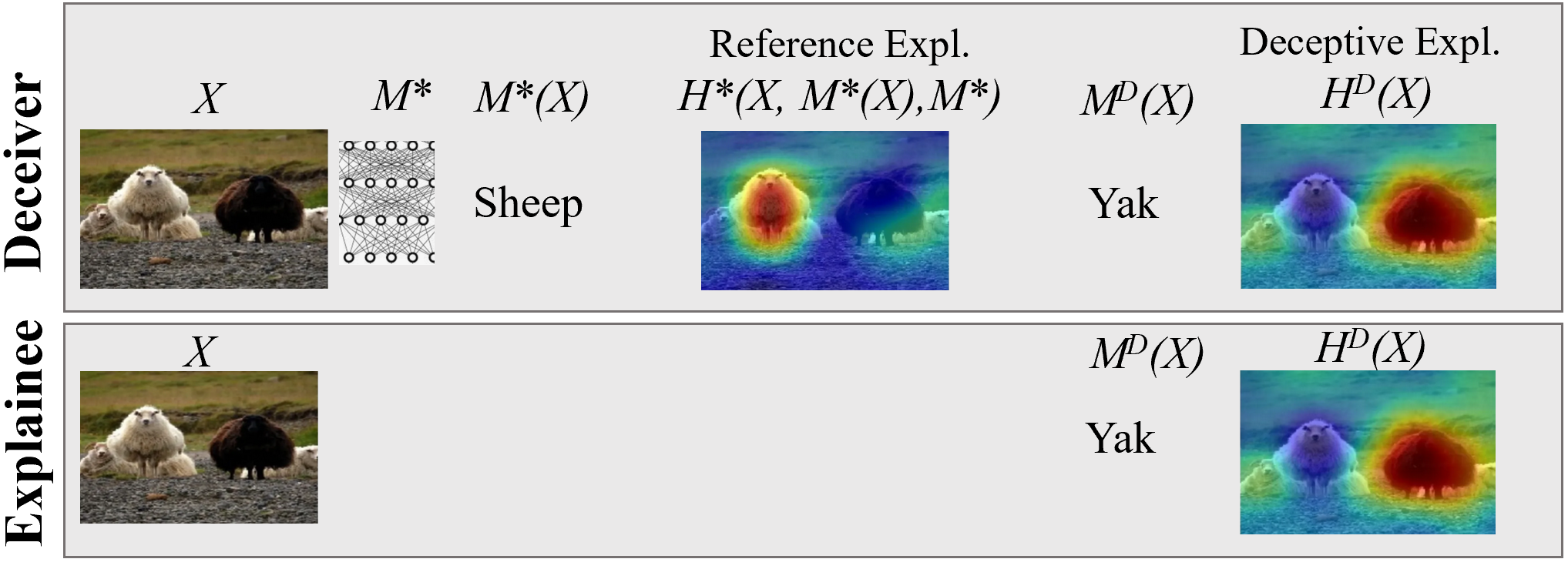}
  \caption{Inputs and outputs for deceiver and explainee for scenario FT in Figure~\protect\ref{fig:scenarios}. \tiny{Images by \protect\cite{pet18}.}} 
  \label{fig:ftexamples}
\end{figure*}

\paragraph*{Creation:}
To construct deceptive explanations (and decisions), a deceiver has access to the model $M^*$ and $M^D$, the input $X$ and the reference explanation $H^*$. She outputs a decision $M^D(X)$ in combination with an explanation $H^D(X)$ (see Figure~\ref{fig:ftexamples}). Deceptive explanations are constructed to maximize the explainee's credence of decisions and explanations. We assume that an explainee is most confident that the reference explanation $H^*(X,Y,M^D)$ and the model-based decision $Y=M^*(X)$ are correct. This encodes the assumption that the truth is most intuitive since any deception must contain some reason that can be identified as faulty. 

We provide simple means for creating deceptive explanations that are non-truthful explanations (FT and FF). The idea is to alter reference explanations. This approach is significantly simpler than creating deceptive explanations from scratch using complex algorithms as done in other works \cite{aiv19,lakkaraju2019fool,adelani2019generating}, while at the same time guaranteeing high-quality deceptive explanations since they are based on what the explainee expects as a valid explanation. For non-truthful explanations a deceiver aims at over-, understating or omitting features $X'\subseteq X$ that are problem or instance-specific. To obtain non-truthful explanations we alter reference explanations in two ways:

\begin{definition}[Omission] \label{def:omit} 
Remove a fixed set of values $\mathcal{V}$ so that no feature $i$ has a value $x_i \in \mathcal{V}$ as follows:
\begin{small}
\begin{equation}
  H_{Omit}(X)_i:=
  \begin{cases}
    0, & \text{if } x_i \in \mathcal{V}. \\
    H^*(X)_i, & \text{otherwise}.
  \end{cases}
\end{equation}
\end{small}
\end{definition}
In our context, this means denying the relevance of some words $\mathcal{V}$ related to concepts such as gender or race. The next alteration distorts relevance scores of all features, eg. to prevent re-engineering through obfuscation.
\begin{definition}[Noise addition] \label{def:noi}
Add noise in a multiplicative manner for any explanation $H^*(X)$:
\begin{small}
\begin{equation}
  H_{Noise}(X)_i:=H^*(X)_i\cdot (1+r_{i,X}),
\end{equation}
\end{small}
where $r_{i,X}$ is chosen uniformly at random in $[-k,k]$ for a parameter $k$ for each feature $i$ and input $X \in S$.
\end{definition}
We assume that these alterations are applied consistently for all outputs. Note, that this does not imply that all explanations are indeed non-truthful, e.g., for noise it might be that by chance explanations are not altered or only very little, for omission it might be that a feature is not relevant in the decision for a particular input $X$, i.e. the value of a feature $H^*(X)_i$ is zero anyway.

\section{\uppercase{Deception Detection}} \label{sec:emp} 

To detect deception attempts, we reason using explanations and decisions of multiple inputs. That is, for a set of inputs $X \in S^D$, we are given for each input $X$ the reported decision $M^D(X)$ and accompanying explanation $H^D(X)$. 
Our goal is to identify if a model outputs deceptive explanations or not. For supervised learning, we (even) aim to identify the inputs yielding deceptive outputs.
We assume that only features that are claimed to contribute positively to a decision are included in explanations. Features that are claimed to be irrelevant or even supporting of another decision outcome are ignored. The motivation is that we aim at explanations that are as simple to understand as possible. The omission of negatively contribution features makes detection harder. We first provide theoretical insights before looking into practical detection approaches.

\paragraph{Formal investigation:} Ideally, any of the three types of deception $\{TF,FT,FF\}$ is detected using only one or more inputs $X \in S^D$ and their responses $M^D(X)$ and $H^D(X)$ (see Figure~\ref{fig:ftexamples}). But, without additional domain knowledge (such as correctly labeled samples), metadata or context information, this is not impossible for all deception attempts. This follows since data, such as class labels, bear no meaning on their own. Thus, any form of "consistent" lying is successful, eg. always claiming that a cat is a dog (using explanations for class dog) and a dog is a cat (using explanations for class cat) is non-detectable for anybody lacking knowledge on cats and dogs, i.e., knowing what a cat or a dog is.

\begin{theorem} \label{thm:det}
There exist non-truthful reported decisions $M^D(X)\neq M^*(X)$ that cannot be identified as non-truthful.
\end{theorem}

\begin{proof}
Consider a model $M^D$ for dataset $\{(X,Y)\}$ for binary classification with labels $Y\in \{0,1\}$ and $M^D(X)=M^*(X)$. Assume a deceiver switches the decision of model $M^D$, i.e. it returns $M^D(X)= 1- M^*(X)$ and $H^*(X,M^D(X),M^D)$. Consider a dataset with switched labels, i.e. $\{(X,1-Y)\}$ and a second model $M'^D$ that is identical to $M^D$ except that it outputs $M'^D(X)=1-Y=1-M^D(X)$. Thus, reference explanations are identical, i.e. we have $H^*(X,M'^D(X),M'^D)= H^*(X,M^D(X),M^D)$. Thus, for input $X$ both the deceiver and model $M'^D$ report $M^D(X)=1-M^D(X)$ and $H^*(X,M^D(X),M^D)$. Therefore, $M'^D$ and $M^D$ cannot be distinguished by any detector. 
\end{proof}

 A similar theorem might be stated for non-truthful explanations $H\neq H^*$, eg. by using feature inversion $H(X)=-H^*(X)$.


The following theorem states that one cannot hide that a feature (value) is influential if the exchange of the value with another value leads to a change in decision. 

\begin{theorem} \label{thm:omit}
Omission of at least one feature value $v \in \mathcal{V}$ can be detected, if there are instances $X, X' \in S$ with decisions $M^D(X)\neq M^D(X')$ and $X'=X$ except for one feature $j$ with $x_j,x'_j \in \mathcal{V}$ and $x'_j\neq x_j$. 
\end{theorem}
\begin{proof} 
We provide a constructive argument. We can obtain for each input $X \in S$, the prediction $M^D(X)$ and explanation $H^D(X)$. By Definition of Omission, if feature values $\mathcal{V}$ are omitted it must hold $H^D(X)_i=0$ for all $(i,X)\in \mathcal{F}_{S,v}$ and $v \in \mathcal{V}$. Omission occurred if this is violated or there are $X, X' \in S$ that differ only in the value $x_j \in \mathcal{V}$ for feature $j$ and $M^D(X)\neq M^D(X')$. The latter holds because the change in decision must be attributed to the fact that of $x_j\neq x'_j$, since $X$ and $X'$ are identical except for feature $j$ with values that are deemed omitted.
\end{proof}

Theorem \ref{thm:omit} is constructive, meaning that it can easily be translated into an algorithm by checking all inputs $S$ if the stated condition is matched. But, generally all inputs $S$ cannot be evaluated due to computational costs. Furthermore, the existence of inputs $X, X' \in S$ that only differ in a specific feature is not guaranteed. However, from a practical perspective, it becomes apparent that data collection helps in detection, i.e. one is more likely to identify "contradictory" samples $X,X'$ in a subset $S' \subset S$ the larger $S'$ is.

\paragraph{Detection Approaches} 
Our formal analysis showed that only decisions and explanations are not sufficient to detect deception involving flipped classes. That is, some knowledge on the domain is needed. Encoding domain know-how with a labeled dataset seems preferable to using expert rules or the-like. Thus, not surprisingly, this approach is common in the literature, e.g. for fake news detection \cite{per17,prz20}. To train a detector, each sample is a triple ($X$, $M^D(X)$, $H^{D}(X)$) for $X \in S^T$ together with label $L \in \{TT, FT, TF,FF\}$ stating the scenario in Figure~\ref{fig:scenarios}. After the training, the classifier can be applied to the explanations and decisions of $X \in S^D$ to investigate. We develop classifiers maximizing deception detection accuracy. 

Labeling data might be difficult since it requires not only domain knowledge on the application but also knowledge on ML, ie. the reference model and explainability method. Thus, we also propose unsupervised approaches to identify whether a model, ie. its explanations, are truthful to the model decision. That is, the goal is to assess if given explanations $H^D$ are true to the model $M^D(X)$ or not.

\begin{algorithm}
\caption{ConsistencyChecker} \label{alg:gui}
\begin{algorithmic}
\STATE \textbf{Input}: Untrained models $\mathcal{M}'$, reference method $H^*$, inputs $S^D$ with (deceptive) decisions and explanations $\{(M^D(X),H^D(X)\}$
\STATE \textbf{Output}: (Outlier) Probability $p$
\STATE $S^{M'}= s$ randomly chosen elements from $S^D$ with $s$ random in $[c_0|S^D|,|S^D|]$ \COMMENT{\emph{\small{We used: $c_0=0.33$}}}
\STATE \small{Train each model $M' \in \mathcal{M}'$ on $(X,M^D(X))$ for $X \in S^{M'}$}
\STATE $m^*_i(X)=\frac{1}{|\mathcal{M}'|} \sum_{M' \in \mathcal{M}'} H_i^*(X,M^D(X),M')$
\STATE $s(M') = $\tiny $\dfrac{ \sum_{i \in[0,n-1],X \in S^D} (H_i^*(X,M^D(X),M')-m_i(X))^2}{n|S^D|}$
\STATE \small $s(M^D)=\dfrac{ \sum_{i \in[0,n-1],X \in S^D} (H_i^D(X)-m_i(X))^2}{n|S^D|}$
\STATE $\mu = \frac{1}{|\mathcal{M}'|}\sum_{M' \in \mathcal{M}'} s(M')$ 
\STATE $\sigma = \frac{1}{|\mathcal{M}'|}\sqrt{\sum_{M' \in \mathcal{M}'} (s(M')-\mu)^2}$
\STATE $p=prob\big(T>|s(M^D)-\mu|\text{\phantom{a}} \big|\text{\phantom{a}} T \sim \mathcal{N}(0,\sigma)\big)$
\end{algorithmic}
\end{algorithm}

Our approach is to check, whether the explanations of $H^D$ and decisions of $M^D$ are consistent. This would be easy, if the model $M^D$ was available, ie. we would check if $H^*(X,M^D(X),M^D)=H^D(X)$. Since it is not, we aim to use a model $M'$ to approximate model $M^D$ and compare the explanations $H^*$ of $M'$ with $H^D$. Since approximation introduces an error, we must determine if differences in the explanations originate from model approximation or from deception. To do so, we train (approximate) reference models $M' \approx M^D$  with $M' \in \mathcal{M}'$ using the provided data $(X,M^D(X))$ with $X \in S^D$. The models might differ, e.g. in hyperparameter settings. If explanations of $M^D$ are "outliers" relative to those of models $ M' \in \mathcal{M}' $ explanations might be untruthful and further investigation is needed. Otherwise, any variation of explanation compared to $H^*$ due to deception is comparable to variation that might arise if models are slightly varied but no deception is attempted.\\
Details on the model $M^D$ are unknown. But for many common ML tasks excellent performing architectures are publicly known. Therefore, we assume that we can anticipate at least a well-working architecture for the problem -- the closer it is to $M^D$ the better.\\
Algorithm \ref{alg:gui} (\emph{ConsistencyChecker})  shows details. It takes as input a set of untrained models $ M' \in \mathcal{M}' $ -- ideally  similar to $M^D$. The more information is known on $M^D$, the more similar the models can be chosen. If only very little information or no information on $M^D$ is available, one might train models on $(X,M^D(X))$ for $X \in S^D$ and choose those that perform best. 

The second unsupervised approach Algorithm \ref{alg:uch} (\emph{UsefulnessChecker}) tries to reproduce the decisions based on explanations. If this is not possible for explanations from $M^D$ but for those of approximate models $M^D$ then explanations from $M^D$ are likely deceptive. As before, we train multiple approximate models $ M' \in \mathcal{M}' $. For each of these models $M'$, we train a classifier $C^{M'}$  using explanations from the approximate models $\mathcal{M'}$ as well as one on explanations from $M^D$. We use the same classifier architecture for all. We conduct a statistical test (as in Algorithm \ref{alg:gui}), if accuracy is an outlier. The full pseudo-code is similar to Algorithm \ref{alg:gui}. For the sake of completeness, it is shown in Algorithm \ref{alg:uch}.

\begin{algorithm}
\caption{UsefulnessChecker} \label{alg:uch}
\begin{algorithmic}
\STATE \textbf{Input}: Untrained models $\mathcal{M}'$, reference method $H^*$, inputs $S^D$ with (deceptive) decisions and expl. $\{(M^D(X),H^D(X)\}$, untrained classifier model $C$
\STATE \textbf{Output}: (Outlier) Probability $p$
\STATE $S^{M'}= s$ randomly chosen elements from $S^D$ with $s$ random in $[c_0|S^D|,|S^D|]$ \COMMENT{We used: $c_0:=0.33$}
\STATE Train each model $M' \in \mathcal{M}'$ on $(X,M^D(X)$ for $X \in S^{M'}$

\STATE $S^{T}= $ random subset of $S^D$ of size $c_1|S^D|$ \COMMENT{We used: $c_1:=0.8$}
\STATE $C^{M'}=$ trained classifier model $C$ on $\big(H^*(X,M^D(X),M'),M^D(X)\big)$ for $X \in S^{T}$ and $M' \in \mathcal{M'}$ \STATE $C^{M^D}=$ trained classifier model $C$ on $\big(H^{D}(X,M^D(X),M^D),M^D(X)\big)$ for $X \in S^T$ 
\STATE $Acc(C^{M}):=$ Accuracy of classifier $C^M$ using $X \in S^D \setminus S^{T}$
\STATE $\mu = \frac{1}{|\mathcal{M}'|}\sum_{M' \in \mathcal{M}'} Acc(C^{M'})$ 
\STATE $\sigma = \frac{1}{|\mathcal{M}'|}\sqrt{\sum_{M' \in \mathcal{M}'} (Acc(C^{M'})-\mu)^2}$
\STATE $p=prob\big(T>|Acc(C^{M^D})-\mu|\text{\phantom{a}} \big|\text{\phantom{a}} T \sim \mathcal{N}(0,\sigma)\big)$
\end{algorithmic}
\end{algorithm}

\section{\uppercase{Evaluation}}
We elaborate on two text classification tasks using a convolutional neural network (CNN) for text classification by \cite{kim-2014-convolutional} as our reference model $M^*$ and GradCAM~\cite{selvaraju2017grad} for generating reference explanations $H^*$. The CNN is well-established, conceptually simple and works reasonably well. GradCAM was one of the methods said to have passed elementary sanity checks that many other methods did not~\cite{adebayo2018sanity}. While GradCAM is most commonly employed for CNN on image recognition the mechanisms for texts are identical. In fact, \cite{lertvittayakumjorn2019human} showed that GradCAM on CNNs similar to the one by \cite{kim-2014-convolutional} leads to outcomes on human tasks that are comparable to other explanation methods such as LIME. The GradCAM method, which serves as reference explanation $H^*$, computes a gradient-weighted activation map starting from a given layer or neuron within that layer back to the input $X$. We apply the reference explanation method $H^*$, ie. GradCAM, on the neuron before the softmax layer that represents the class $Y'$ to explain. For generating a high fidelity explanation for an incorrectly reported prediction $M^D(X)\neq M^*(X)$ (scenario FT in Figure~\ref{fig:scenarios}) we provide as explanation the reference explanation, i.e. $H^D(X)=H^*(X,M^D(X),M^D)$. By definition reference explanations maximize explanation fidelity $f_O$.

\begin{figure}
  \centering
  \includegraphics[width=1.03\linewidth]{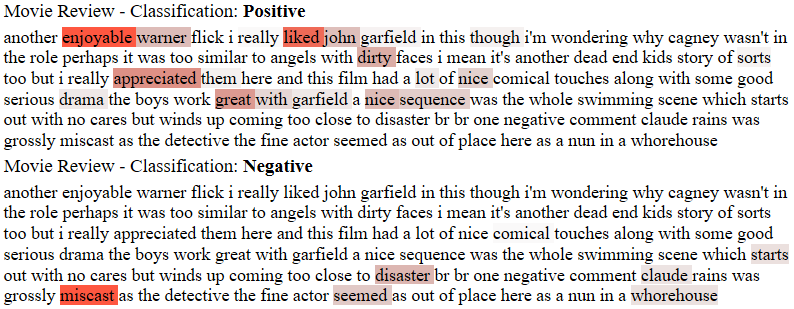}
  \caption{Generated sample explanations for scenarios TT (top) and FT (bottom) from Figure~\ref{fig:scenarios}}
  \label{fig:gradcamsample}
\end{figure}
 
\paragraph{\textbf{Setup:}} 
We employed two datasets. The IMDB dataset~\cite{imdb} consists of movie reviews and a label indicating positive or negative sentiment polarity. We also utilized the Web of Science (WoS) dataset consisting of abstracts of scientific papers classified into 7 categories~\cite{wos}. Our CNNs for classification achieved accuracies of 87\% for IMDB and 75\% for WoS trained  with 2/3 of the samples for training and 1/3 for testing. We computed explanations for test data only. For deception using omission, we removed a randomly chosen set of words $V$ (see Def. \ref{def:omit}), such that their overall contribution to all explanations $H^*$ is $k$\% (with a tolerance of $0.01k$ \%). The contribution of a word $v$ is given by $\sum_{(i,X) \in \mathcal{F}(v,S)} H^*(X,M^*(X))_i$. For explanation distortion parameter $k$ (see Definitions \ref{def:omit} and \ref{def:noi}) we state values for each experiment.\\

\paragraph{\textbf{ML based Detection:}} 
As detector models, we used CNN models.  
For supervised learning, the model input is a concatenation of three vectors: i) a text vector of word indices, ii) a heatmap vector of values obtained via GradCAM, that is a 1:1 mapping of the visual output shown to the user, and iii) a one-hot prediction vector of the decision. Our "simple" CNN detector, i.e. classifier, is designed as follows: we perform an embedding, concatenate the heatmap vector with the word embedding before doing a 1D convolution. Then we concatenate the one-hot prediction vector and use two dense layers. The more "complex" CNN adds six more conv1D layers: two processing the embedding, two on the heatmap vector, and two after the first concatenation. We used dropout for regularization. Since labeling is difficult and potentially error-prone, we consider different levels of label noise, i.e., $L \in [0,0.32]$, such that a fraction $L$ of all labels were replaced with a random label (different from the correct one). For the detection experiment, we chose samples that were predicted correctly by the truthful model.  
For unsupervised learning, we train 35 classifiers $ M' \in \mathcal{M}' $ being variations of a CNN network \cite{kim-2014-convolutional}, i.e., each of the following hyperparameters was chosen uniformly at random for each classifier $M'$:  embedding dimension $\{32,64, 128\}$; 1-3 linear layers; 2-6 conv layers for the Kim network  with varying number of filters. We also varied the training sets in terms of size and elements, ie. we trained a model with a subset of $\mathcal{T}$ of size 33, 50 and 100\%. Any model was trained using the Adam optimizer for 100 epochs. Train/Test data split was 80/20 for all detector models.

Classifiers learning from (deceptive) explanations as done in our unsupervised approach \emph{UsefulnessChecker} tend sometimes to focus on raw inputs $X$ and disregard explanation relevance scores $H^D_i(X)$. That is, they often work well and show little variation in accuracy despite large variations in explanations.  To avoid this, we convolve also an inner representation of the network with explanation values enforcing stronger entanglement. That is, in the \emph{UsefulnessChecker} model the output of the word embedding of the input is convolved with the explanations as follows: First, we perform a low-dimensional embedding (just one dimensional) and multiply the embedding values with the explanation values and add explanation values on top. This is then fed into 3 Conv1D layers followed by two dense layers. 

\paragraph{\textbf{Human-based Detection:}} 
We conducted a user study using the IMDB dataset.\footnote{The WoS dataset seems less suited since it uses expert terminology that is often not held by the general public from which participants originate as found in~\cite{lertvittayakumjorn2019human}.} For the scenarios of interest, we compare explanations that are aligned to the shown prediction, i.e. TT and FT. Two samples are shown in Figure~\ref{fig:gradcamsample}. We recruited a total of 200 participants on Amazon Mechanical Turk from the US having at least a high-school degree. We presented each participant with 25 predictions together with explanations. They had to answer "Is the classification correct?" on a scale of five ranging from strongly disagree to  strongly agree. We randomized the choice of presented samples, i.e. we randomly chose a sample of the dataset and we randomly chose between scenarios TT and FT in Figure~\ref{fig:scenarios}.

\section{\uppercase{Results}}

\paragraph{\textbf{Human-based Detection:}} 
Out of the 200 participants, we removed participants that spend less than 5 seconds per question, since we deemed this time too short to provide a reasonable answer. We also filtered out participants who always gave the same answer for all 25 questions. This left 140 participants amounting to 3,500 answers. Demographics and answer distributions are in Figure \ref{fig:plotuser1} and Table \ref{tab:samp_demo}. 

A t-test of means confirmed that the distributions differ significantly (p-value of 0.008), though the mean scores for "agreeing" of  3.74(TT) and 3.58(FT) show that in absolute terms differences are minor. This implies that while the majority of humans might be fooled oftentimes, they have the capability to collectively detect deceptive explanations.

\begin{table}[!t]
\scriptsize
\centering
\begin{tabular}{|l|l|l|}
\hline
\multicolumn{1}{|c}{Variable} & \multicolumn{1}{|c|}{Value} & \multicolumn{1}{c|}{Percentage }\\
\hline
Gender & Male & 66\%\\
 & Female & 34\%\\ \hline
Age & $\leq$ 25 years of age & 18\%\\
 & from 26 to 40 years of age & 62\%\\
 & from 41 to 65 years of age & 18\%\\
 & $>$ 65 years of age & 2\%\\ \hline
Education & High School & 16\%\\
 & Associate Degree & 11\%\\
 & Bachelor's Degree & 56\%\\
 & Master's Degree & 16\%\\
 & Doctoral Degree & 1\%\\
\hline
\end{tabular}
\caption{Participants Demographics with $n=140$ participants} \label{tab:samp_demo}
\end{table}

\begin{figure}
  \centering
  \includegraphics[width=0.95\linewidth]{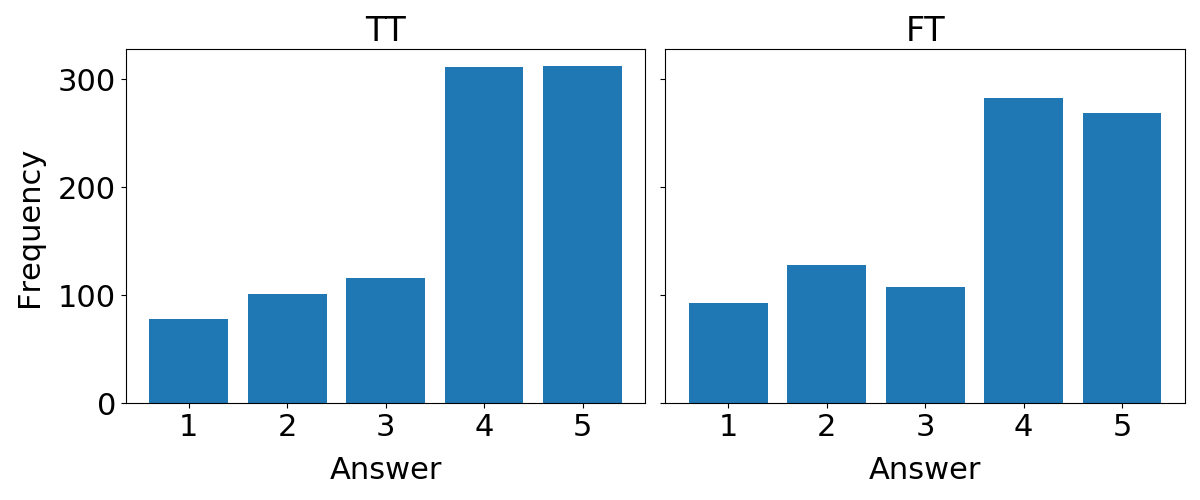}
  \caption{Distributions of user replies to ``The classification is correct'' \scriptsize{(1 = strongly disagree to 5 = strongly agree)}.}
  \label{fig:plotuser1}
\end{figure}

\begin{figure*}[!htbp]
  \centering
  \includegraphics[width=1.05\linewidth]{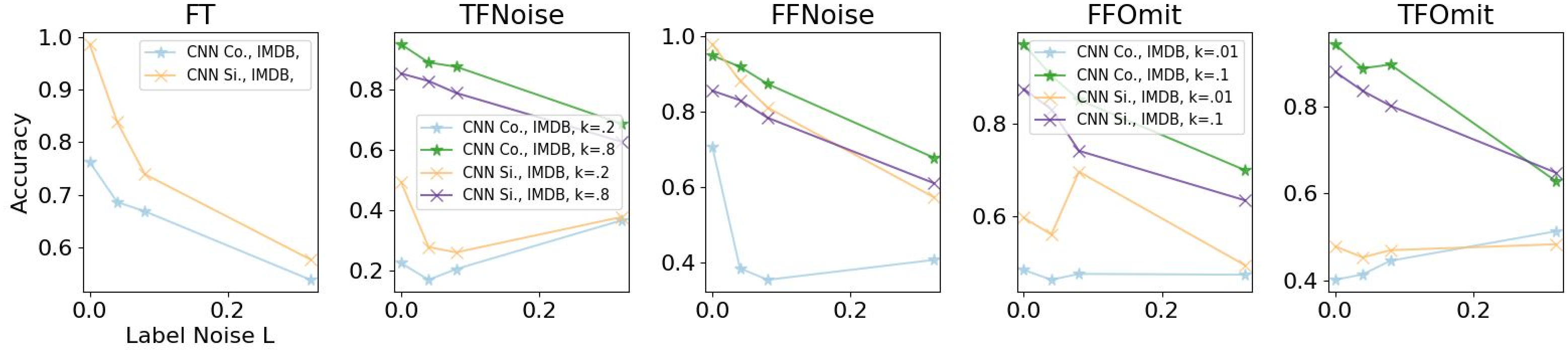}
  \caption{Supervised detection results for IMDB for scenarios in Figure \ref{fig:scenarios}.}
  \label{fig:plotmachine}
\end{figure*}

\begin{figure*}[!htbp]
  \centering
  \includegraphics[width=1.05\linewidth]{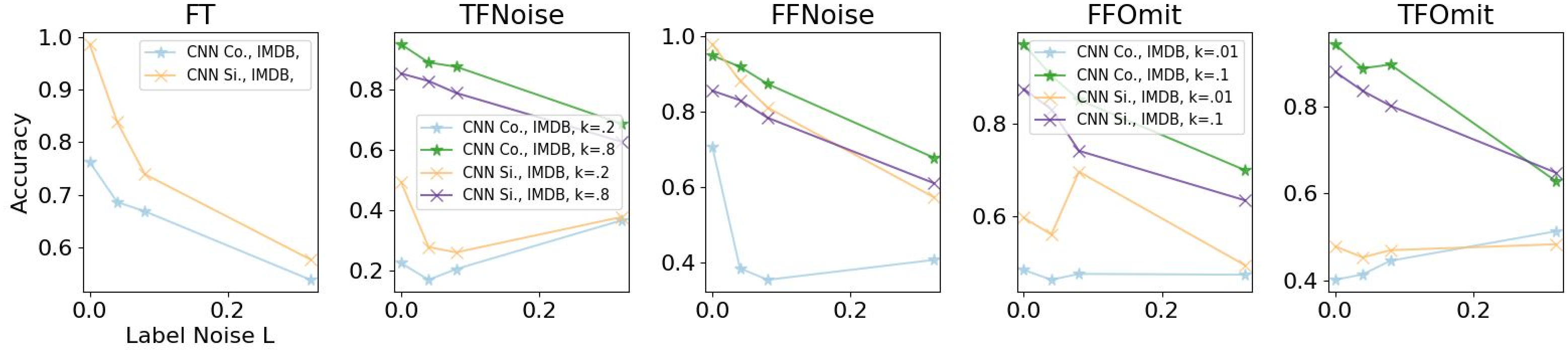}
  \caption{ML-based supervised detection results for WoS for scenarios in Figure \ref{fig:scenarios}.}
  \label{fig:plotmachine2}
\end{figure*}

\paragraph{\textbf{Supervised ML-based Detection }}

Figure \ref{fig:plotmachine} shows results for IMDB, those for WoS are similar, shown in Figure \ref{fig:plotmachine2}. The plot shows that for very minor deception (small $k$) detection is more difficult. Accuracy rarely tops 60\%, which is fairly low given that 50\% is the baseline for guessing. Accuracy is also more volatile given label noise. Beyond some threshold under all scenarios, accuracy exceeds 90\% for correctly deciding whether a decision and/or explanation are deceptive. Accuracy worsens with greater label noise, though it is often better than guessing even for a noise level of 32\%. Mostly the complex CNN (with more layers) outperforms the simple one with one exception: the scenario where the explanation is always truthful to the model, but the decision is not. Here, a simple function allows distinguishing deceptive (FT) from the non-deceptive scenario (TT). We observed that for the scenario with noise on non-truthful explanations for correct decisions (TFNoise) small noise levels can lead to very low test accuracies. Training accuracies were larger, but mostly close to random guessing, indicating that the network fails to learn proper conditions what values constitute a noisy relevance score and which do not.


\paragraph{\textbf{Unsupervised ML-based Detection }}
For a single model, a small probability indicates that it is unlikely to witness such an (extreme) outcome, i.e. accuracy or deviation from the average explanation. Thus, a small probability such as 0.05 indicates a high likelihood of deception. Figure \ref{fig:unsup} shows the average as well as maximum probability of all 35 models.  Detection ability follows an S-curve, meaning that for deception with parameterized strength $k$ above a threshold detection works fairly well, whereas it fails otherwise.  Required noise levels are somewhat higher for the WoS dataset, as shown in Figure \ref{fig:plotmachine2} but qualitatively, results are identical. The trained classifier performs worse on WoS than IMDB, since WoS is a more complex dataset. Classifiers vary (relatively) more in accuracy, which in turn also leads to more variance in explanations. Looking at the mean and max indicates that for multiplicative noise (TFNoise) all deceptive models are reliably detected for $k=0.8$, while below that value at least for the most difficult to detect model it fails to do so.  For $k<0.8$, the detection methods are still valuable to identify suspicious models, meaning that such models exhibit lower probability, but are not low enough to be certain. The same reasoning also applies to TFOmit, though here a strong difference between methods is apparent. The \emph{ConsistencyChecker} yields much better outcomes, highlighting that even small omissions can be detected reliably. It shows that statistical analysis is preferable to using a downstream task. Our models $\mathcal{M'}$ are very diverse, ie. models differ by a factor of 3 in terms of training data and number of layers, as well as in neurons/filters per layer. We found that reducing (or increasing) the diversity has a profound impact on results, as shown in Figures \ref{fig:unsuperData} and \ref{fig:unsuperHyper}. 

\begin{figure}[!htbp]
  \centering
  \includegraphics[width=1.04\linewidth]{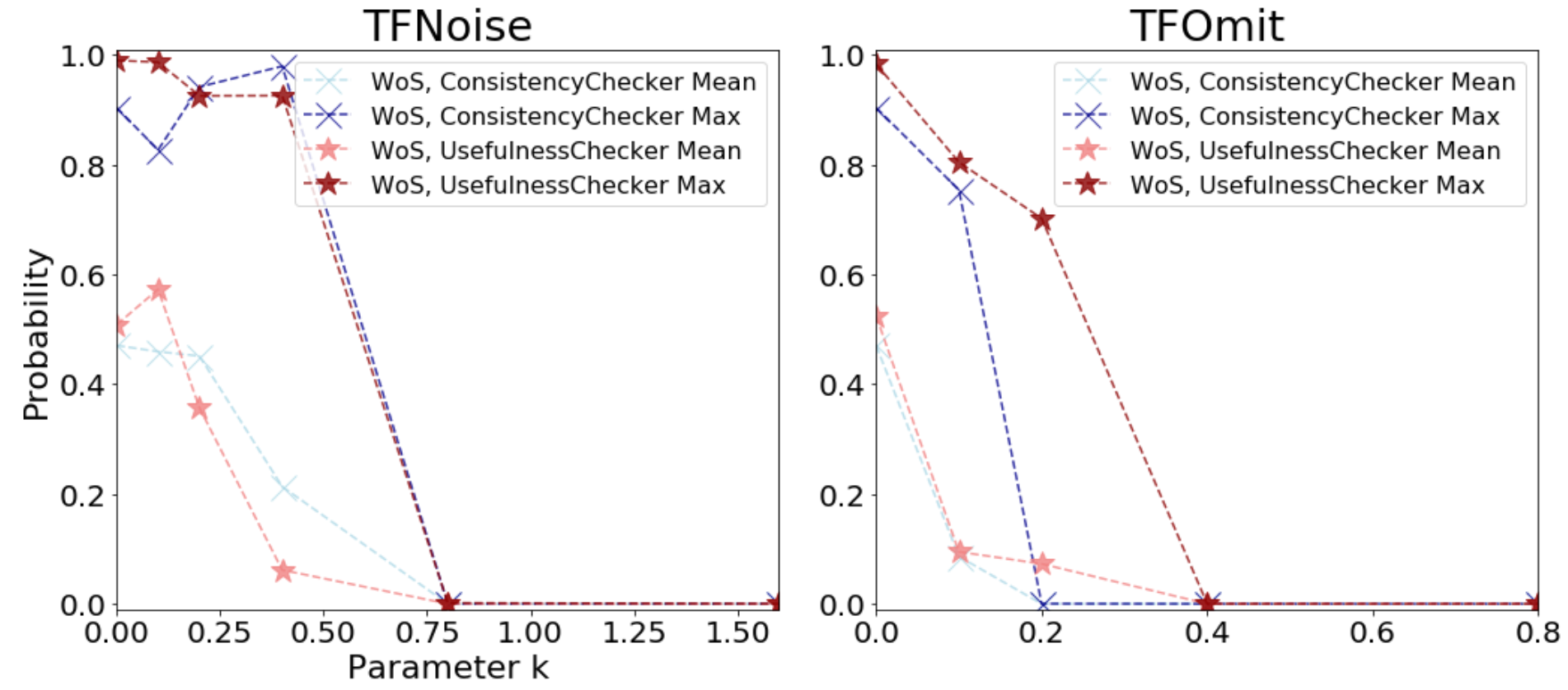}
  \caption{Unsupervised detection results for WoS where approximate models vary only in training data (but have the same hyperparameters) }
  \label{fig:unsuperData}
\end{figure}

\begin{figure}[!htbp]
  \centering
  \includegraphics[width=1.04\linewidth]{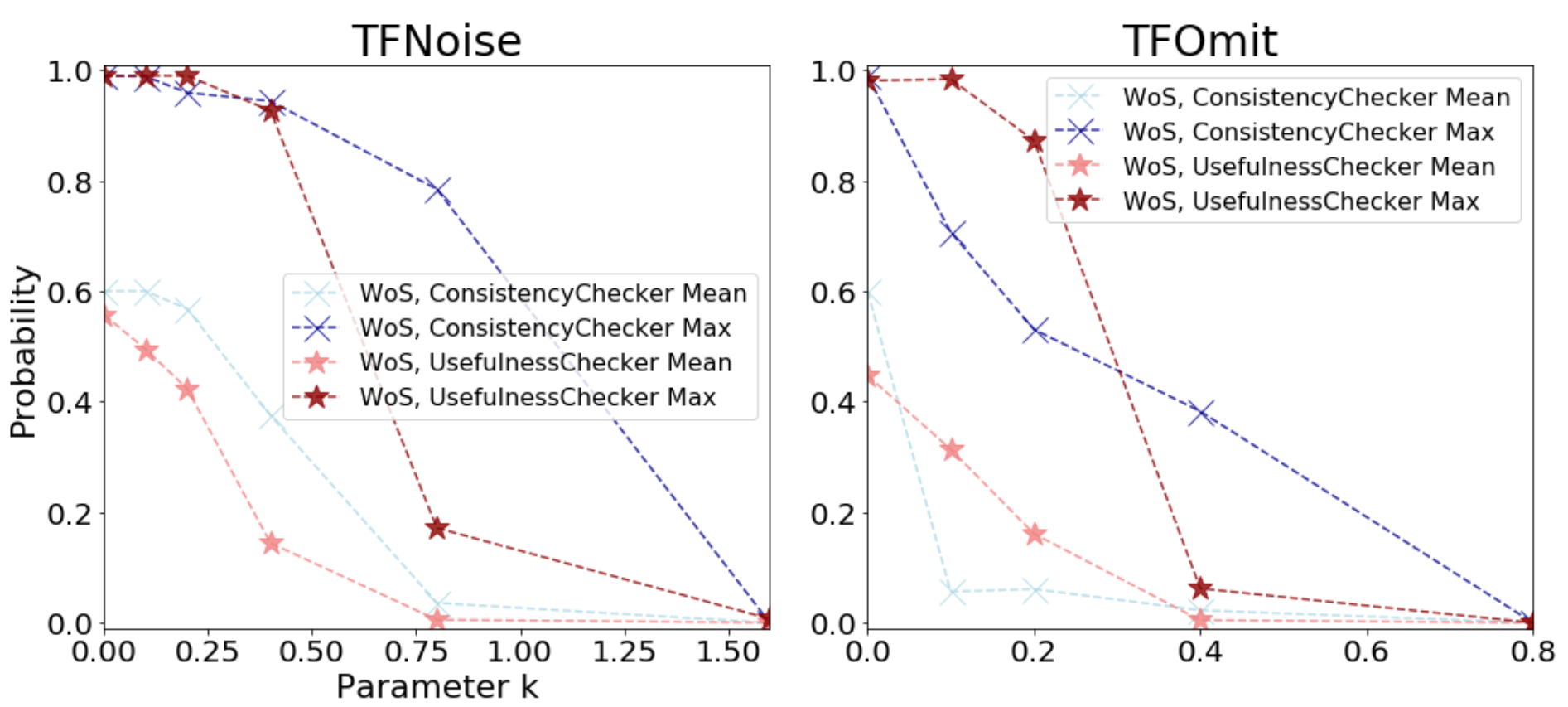}
  \caption{Unsupervised detection results for WoS where approximate models vary in training data and hyperparameters. Detection is more difficult compared to varying training data only (Figure \ref{fig:unsuperData}) }
  \label{fig:unsuperHyper}
\end{figure}
\section{\uppercase{Difficulty of Deception Detection}}
We provide intuition for Algorithm \emph{ConsistencyChecker} discussing the difficulty of detection depending on noise models and deception strategy. To compute the probability, we rely on values $s(M)$ as defined in Algorithm \ref{alg:gui}. We are interested in the gap $G(H_i,m_i(X)):=E[(H_i-m_i(X))^2]$ between the mean and the relevance score in the explanation of a feature $i$. For multiplicative noise we have $H_i^D(X,M^D(X),M)=(1+U)\cdot H_i^*(X,M^D(X),M^D)$, where $U$ is uniformly chosen at random from $[-k,k]$. We shall use $a_i:=H_i^*(X,M^D(X),M^D)$ and $m_i:=m_i(X)$ for ease of notation.  We expect that the deviation for a deceptive explanation and the mean is:
\begin{footnotesize}
\begin{align}
G((1+U)a_i,m_i)= & E[((1+U)a_i-m_i)^2]  \nonumber\\
        &= E[(a_i-m_i)^2-2Ua_im_i-U^2a_i^2] \nonumber \\
                &= (a_i-m_i)^2+a_i^2k^2/3\nonumber
\end{align}
\end{footnotesize}
The overall deviation $s(M)$ for a model is just the mean across all features $i$ and inputs $X$. Detection is difficult when 
\begin{footnotesize}
\begin{align}
\sum_{i,X} (a_i-m_i)^2\gg \sum_{i,X} a^2k^2/3 \label{eq:int}
\end{align}
\end{footnotesize}
Put in words, detection is difficult, when  the distortion due to deception (right-hand side term in Equation \ref{eq:int}) is small compared to the one due to model variations $\mathcal{M'}$ (left hand side term in Equation \ref{eq:int}). The closer $a_i$ and $m_i$  are and the larger $k$, the easier detection. 
For omission we get that if feature $i$ is omitted then $G(a_i,m_i)=m_i^2$ and $G(a_i,m_i)=(a_i-m_i)^2$. Assume a set $F^D$ of features is omitted, where the size of $F^D$ depends on the parameter $k$. We get that deception is difficult if  $\sum_i (a_i-m_i)^2 \gg \sum_{i \in F^D} m_i^2 + \sum_{i \notin F^D} (a_i-m_i)^2$. Clearly, the larger $F^D$ the easier detection. Say we omit features with $m_i=a_i$ and we are given the choice of omitting two features with mean $m$ or one with mean $2m$. The latter is easier to detect since means are squared, ie. $m^2+m^2=2m^2 < (2m)^2=4m^2$. Therefore, it is easier to detect few highly relevant omitted features than many irrelevant ones.


\begin{figure}[!htbp]
  \centering
  \includegraphics[width=1.04\linewidth]{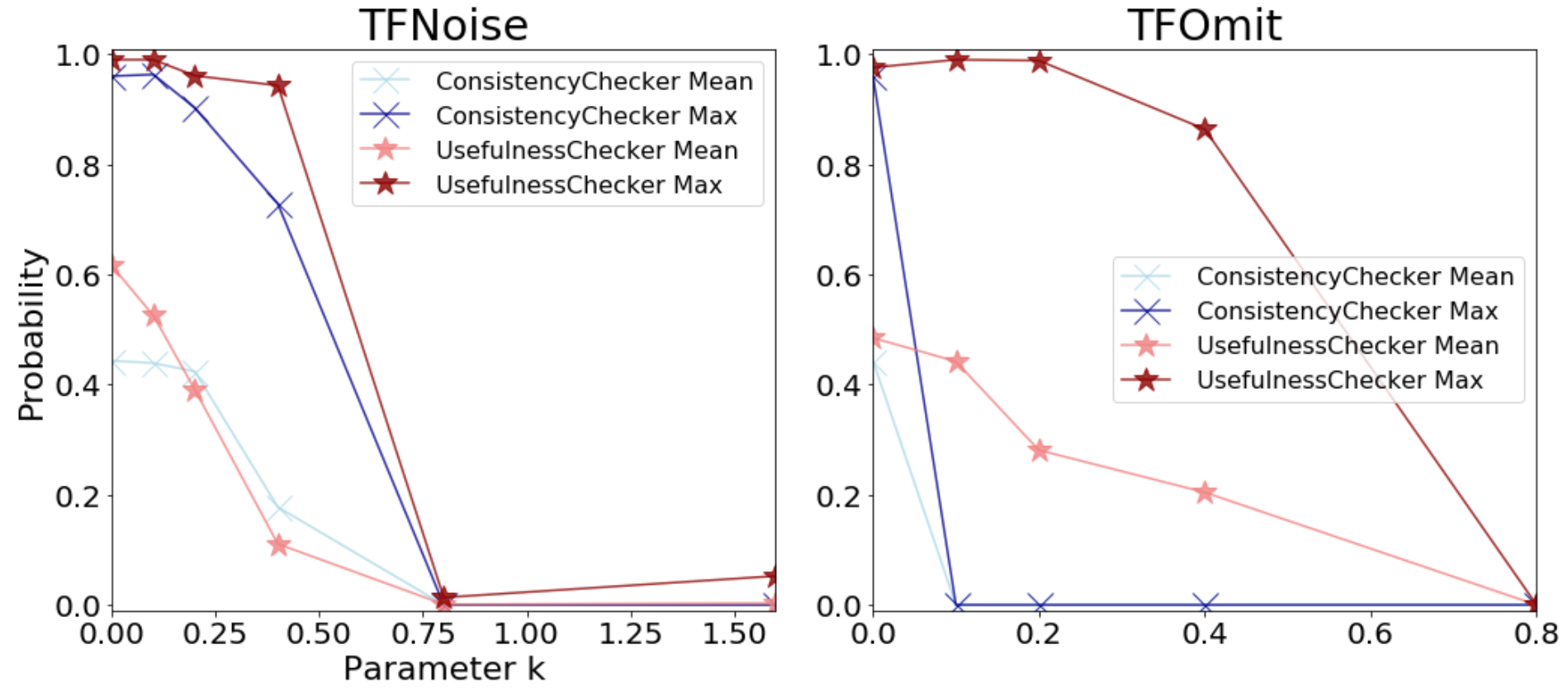}
  \caption{Unsupervised detection results for IMDB.}  \label{fig:unsup}
\end{figure}

\section{\uppercase{Related Work}} \label{sec:typ} 
\cite{sla20} showed how arbitrary explanations for methods relying on perturbations can be generated for instances by training a classifier with adversarial inputs. \cite{dima20} trains a classifier using an explainability loss term for a feature that should be masked in explanations. \cite{fuk20} showed that biases in decision-making are difficult to detect in an input-output dataset of a biased model if the inputs were sampled in a way to disguise the detector. \cite{lai19} used ML (including explanations) to support detection of deceptive content. The explanations were non-deceptive.

\cite{viering2019manipulate} are interested in manipulating the inner workings of a deep learning network to output arbitrary explanations. Whether the explanations themselves are convincing or not, is not considered, i.e., the paper shows many examples of "incredible" explanations that can easily be detected as non-genuine. \cite{aiv19} focus on manipulating reported fairness based on a regularized rule list enumeration algorithm. \cite{lakkaraju2019fool} investigated the effectiveness of misleading explanations to manipulate users' trust. Decisions were made using prohibited features such as gender and race but misleading explanations were supposed to disguise their usage. Both studies~\cite{aiv19,lakkaraju2019fool} found that users can be manipulated into trusting high fidelity but misleading explanations for correct predictions. In contrast, we do not generate fake reviews but only generate misleading justifications for review classifications and provide detection methods and some formal analysis. 

Inspiration for detecting deceptive explanations might be drawn from methods used for evaluating the quality of explanations \cite{mohs21}. In our setup, quality is a relative notion compared to an existing explainability method and not to a (human) gold standard. 
Papenmeier et al.~\cite{papenmeier2019model} investigated the influence of classifier accuracy and explanation fidelity on user trust. They found that accuracy is more relevant for trust than explanation quality though both matter. 

~\cite{nourani2019effects} investigated the impact of explanations on trust. Poor explanations indeed reduce a user's perceived accuracy of the model, independent of its actual accuracy. Explanations' helpfulness varies depending on task and method~\cite{lertvittayakumjorn2019human}. Explanations are more helpful in assessing a model's predictions compared to its behavior. Some methods support some tasks better than others. For instance, LIME provides the most class discriminating evidence, while the layer-wise relevance propagation (LRP) method~\cite{bach2015pixel} helps assess uncertain predictions.

~\cite{adelani2019generating} showed how to create and detect fake online reviews of a pre-specified sentiment. In contrast, we do not generate fake reviews but only generate misleading justifications for review classifications. Fake news detection has also been studied\cite{per17,prz20} based on ML methods and linguistic features obtained through dictionaries. \cite{per17,prz20} use a labeled data set. Linguistic cues~\cite{lud16} such as flattery was used to detect deception in e-mail communication. We do not encode explicit, domain-specific detection features such as flattery.

Our methods might be valuable for the detection of fairness and bias -- see \cite{meh19} for a recent overview. There are attempts to prevent ML techniques from making decisions based on certain attributes in the data, such as gender or race~\cite{ross2017right} or to detect learnt biases based on representations~\cite{zhang2018examining} or perturbation analysis for social associations~\cite{prab19}. In our case, direct access to the decision-making system is not possible –- neither during training nor during operations, but we utilize explanations. 


In human-to-human interaction, behavioral cues such as response times~\cite{lev14} or non-verbal leakage due to facial expressions~\cite{ekman1969nonverbal} might have some, but arguably limited impact~\cite{masip2017deception} on deception detection. In our context, this might pertain, e.g., to computation time. We do not use such information. Explanations to support deceptions typically suffer from at least one fallacy such as "the use of invalid or otherwise faulty reasoning" ~\cite{van2009fallacies}. Humans can use numerous techniques to attack fallacies~\cite{damer2012attacking}, often based on logical reasoning. Such techniques might also be valuable in our context. In particular, ML techniques have been used to detect lies in human interaction, eg. \cite{aro18}.

\section{\uppercase{Discussion}} 
Explanations provide new opportunities for deception (Figure \ref{fig:scenarios}) that are expected to rise since AI is becoming more pervasive, more creative \cite{schne22}, personalized \cite{schne21pers}. Deceptive explanations might aim at disguising the actual decision process, e.g., in case it is non-ethical, or make an altered prediction appear more credible. While faithfulness of explanations can be clearly articulated mathematically using our proposed decision and explanation fidelity measures, determining when an explanation is deceptive, is not always as clear, since it includes a grey area. That is, an explanation might be said to be deceptive, but it might also only be judged as inaccurate or simplified. Thus, deception detection is not an easy task: While strong deception is well-recognizable, minor forms are difficult to detect. Furthermore, some form of domain or model knowledge is necessary. This could be data similar or, preferably, identical to the model's training data under investigation. Domain experts could also provide information in the form of labeled samples or detection rules, i.e., they can investigate model outputs and judge them as faithful or deceptive. Identifying deceptive explanations becomes much easier if model access and training or testing data are available, i.e., it reduces to comparing outputs from models to those suggested by the (training) data. We recommend regulatory bodies to pass laws that ensure that auditors have actual model access since this simplifies the process of deception detection. It is one step towards ensuring that AI is used for the social good. 

Detection methods will improve, but so will strategies for lying. Thus, it is important to anticipate weaknesses of detection algorithms that deceitful parties might exploit, and mitigate them early on, e.g., with the aid of generic security methods \cite{schl15}.
The field of explainability evolves quickly with many challenges ahead \cite{mesk21}. This provides ample opportunities for future research to assess methods for creation and detection of deceptive explanations, e.g., methods explaining features or layers of image processing systems rather than text \cite{schn21}. 

\section{\uppercase{Conclusion}} 
Given economic and other incentives, a new cat and mouse game between "liars" and "detectors" is emerging in the context of AI. Our work provided a first move in this game: We structured the problem, and contributed by showing that detection of deception attempts without domain knowledge is challenging. Our ML models utilizing domain knowledge through training data yield good detection accuracy, while unsupervised techniques are only effective for more severe deception attempts or given (detailed) architectural information of the model under investigation.

\bibliographystyle{apalike}
{\small
\bibliography{aimisuse}}

\end{document}